\documentclass{article}

\usepackage{amsmath,amssymb, amsthm}
\usepackage{amsfonts}
\usepackage{layout}
\usepackage{amsmath}
\usepackage{amsthm}
\usepackage{amssymb}
\usepackage{url}
\usepackage{color}
\usepackage{graphicx}
\usepackage{verbatim}

\newcommand{\ve}[2]{\langle #1 ,  #2 \rangle}

\newcommand{\pseudoinv}[1]{#1^{-1}}

\newcommand{\R}{\mathbb{R}}
\newcommand{\Prob}{\mathbf{Prob}}

\newcommand{\bA}{\mathbf A}

\newcommand{\bD}{\mathbf D}
\newcommand{\bM}{\mathbf M}

\newcommand{\bI}{\mathbf I}

\newcommand{\bX}{\mathbf X}
\newcommand{\bY}{\mathbf Y}

\newcommand{\bG}{\mathbf G}
\newcommand{\bS}{\mathbf S}

\newcommand{\Sam}{\hat{S}}

\DeclareMathOperator{\support}{{supp}}

\DeclareMathOperator{\Exp}{\mathbb{E}}           

\newtheorem{assumption}{Assumption}
\newtheorem{lemma}{Lemma}
\newtheorem{theorem}{Theorem}
\newtheorem{proposition}{Proposition}

\theoremstyle{plain}

\theoremstyle{definition}

\makeatletter
\newcommand*{\rom}[1]{\expandafter\@slowromancap\romannumeral #1@}
\makeatother

\usepackage{times}
\usepackage{graphicx} 
\usepackage{subfigure} 
\usepackage{epstopdf}
\usepackage{natbib}
\usepackage{algorithm}
\usepackage{algorithmic}

\usepackage{hyperref}

\usepackage[accepted]{icml2015a}

\DeclareMathOperator*{\argmin}{arg\,min}
\def\<#1,#2>{\langle #1,#2\rangle}

\icmltitlerunning{SDNA: Stochastic Dual Newton Ascent for Empirical Risk Minimization}

\begin{document} 

\twocolumn[

\icmltitle{SDNA: Stochastic Dual Newton Ascent for Empirical Risk Minimization}

\icmlauthor{Zheng Qu}{zheng.qu@ed.ac.uk}
\icmladdress{School of Mathematics, University of Edinburgh, UK}
\icmlauthor{Peter Richt\'{a}rik}{peter.richtarik@ed.ac.uk}
\icmladdress{School of Mathematics, University of Edinburgh, UK}
\icmlauthor{Martin Tak\'{a}\v{c}}{Takac.MT@gmail.com}
\icmladdress{Industrial and Systems Engineering, Lehigh University, USA}
\icmlauthor{Olivier Fercoq}{olivier.fercoq@telecom-paristech.fr}
\icmladdress{Telecom Paris-Tech, France}

\icmlkeywords{empirical risk minimization, randomized algorithms, iterative Hessian sketch, coordinate descent, convex optimization, duality}

\vskip 0.3in
]

\begin{abstract}
We propose a new algorithm for minimizing regularized empirical loss: {\em Stochastic Dual Newton Ascent (SDNA)}. Our method is dual in nature: in each iteration we update a random subset of the dual variables. However, unlike existing methods such as stochastic dual coordinate ascent, SDNA  is capable of utilizing {\em all curvature information} contained in the examples, which leads to striking improvements in both theory and practice -- sometimes by orders of magnitude. In the special case when an L2-regularizer is used in the primal, the dual problem is a concave quadratic maximization problem plus a separable term. In this regime, SDNA in each step solves a proximal subproblem involving  a random principal submatrix of the Hessian of the quadratic function; whence the name of the method. If, in addition, the loss functions are quadratic, our method can be interpreted as a novel variant of the recently introduced {\em Iterative Hessian Sketch}. 

\end{abstract}

\section{Introduction}

Empirical risk minimization (ERM) is a fundamental paradigm in the theory and practice of statistical inference and machine learning \cite{UML:book}. In the ``big data'' era it is increasingly common in practice to solve ERM problems with a massive number of examples, which leads to new algorithmic challenges. 

State-of-the-art optimization methods for ERM include i) stochastic (sub)gradient descent \cite{Pegasos-MAPR, Pegasos2}, ii) methods based on stochastic estimates of the gradient with diminishing variance such as SAG \cite{SAG}, SVRG \cite{SVRG}, S2GD \cite{S2GD}, proxSVRG \cite{proxSVRG},   MISO \cite{MISO}, SAGA \cite{SAGA}, minibatch S2GD \cite{mS2GD}, S2CD \cite{S2CD}, and iii) variants of stochastic dual coordinate ascent \cite{SDCA, IProx-SDCA, Pegasos2, ASDCA, APROXSDCA, APCG, QUARTZ, MinibatchASDCA}.



There have been several attempts at designing methods that combine randomization with the use of curvature (second-order) information. For example, methods based on running coordinate ascent in the dual such as those mentioned above and also 
\cite{UCDC, PCDM, SPCDMx, DQA, Hydra, NSync, APPROX, Hydra2, QUARTZ, ALPHA} use curvature information contained in the diagonal of a bound on the Hessian matrix. Block coordinate descent methods, when equipped with suitable data-dependent norms for the blocks, use information contained in the block diagonal of the Hessian \cite{ICD}. 
 
A more direct route to incorporating curvature information was taken by  \citet{Stoch-LBFGS} in their stochastic L-BFGS method, by \citet{StochQNewton} and \citet{SGD-QN} in their stochastic quasi-Newton methods and by \citet{RobustBCD} who proposed a stochastic block coordinate descent methods.   While typically efficient in practice, none of the methods mentioned above are equipped with complexity bounds (bounds on the number of iterations). An exception in this regard is the work of  \citet{SGD-QN-2009}, who give a $O(1/\epsilon)$ complexity bound for a Quasi-Newton SGD method.


\subsection{Contributions}

The main contribution of this paper is the design and  analysis of a new algorithm---{\bf stochastic dual Newton ascent  (SDNA)}---for solving a regularized ERM problem with smooth loss functions and a strongly convex regularizer (primal problem). Our method is stochastic in nature and has the capacity to  utilize {\em all curvature information} inherent in the data. While we do our analysis for an arbitrary strongly convex regularizer, for the purposes of the introduction we shall describe the method in the case of the L2 regularizer.  In this case, the dual problem is a concave quadratic maximization problem with a strongly concave separable penalty. 

{\bf SDNA in each iteration picks a random subset of the dual variables (which corresponds to picking a minibatch of examples in the primal problem), following an arbitrary probability law, and maximizes,  exactly, the dual objective restricted to the random subspace spanned by the coordinates.}  Equivalently, this can be seen as the solution of a proximal subproblem involving a random principal submatrix of the Hessian of the quadratic function. Hence, {\bf SDNA utilizes  all curvature information available in the random subspace in which it operates.}  Note that this is very different from the update strategy of parallel / minibatch coordinate descent methods. Indeed, while these methods  also  update a random subset of variables in each iteration, they instead only utilize curvature information present in the diagonal of the Hessian. 

As we will explain in detail in the  text, SDCA-like methods need {\em more} iterations (and hence more passes through data) to convergence as the  minibatch size increases. However, SDNA enjoys the opposite behavior: {\bf with increasing minibatch size, SDNA needs  fewer iterations (and hence fewer passes over data) to convergence.} This observation can be deduced from the complexity results we prove for SDNA, and is also confirmed by our numerical experiments.  In particular, we show that {\bf the expected  duality gap decreases at a geometric rate} which i) is better than that of SDCA-like methods such as SDCA \cite{SDCA} and QUARTZ \cite{QUARTZ}, and  ii) improves with increasing minibatch size.  This improvement does not come for free: as we increase the minibatch size, the subproblems  grow in size as they involve larger portions of the Hessian. We find through experiments that for some, especially dense problems, {\bf even relatively small minibatch sizes lead to dramatic speedups in actual runtime.}

We show that {\bf in the case of quadratic loss, and when viewed as a primal method, SDNA can be interpreted as a  variant of the recently introduced Iterative Hessian Sketch algorithm} \cite{IteHeSke}.

En route to developing SDNA which we describe in Section~\ref{sec:ERM}, {\bf we also develop several other new algorithms}: two  in Section~\ref{sec:smooth} (where we focus on smooth problems), one in  Section~\ref{sec:proximal} (where we focus on composite problems). Besides SDNA, we also develop and analyze a {\bf novel minibatch variant of SDCA} in Section~\ref{sec:ERM},  for the sake of finding suitable method to compare SDNA to.  SDNA is equivalent to applying the new method developed in  Section~\ref{sec:proximal} to the dual of the ERM problem. However, as we are mainly interested in solving the ERM (primal) problem, we additionally prove that the expected duality gap decreases at a geometric rate. Our technique for doing this is a variant of the one use by \citet{SDCA}, but generalized to an arbitrary sampling.

%

\subsection{Notation} 

\textbf{Vectors.}  By $e_1,\dots, e_n$ we denote the standard basis vectors in  $\R^n$. For any $x\in \R^n$, we denote by $x_i$ the $i$th element of $x$, i.e., $x_i=e_i^\top x$. For any two vectors $x,y$ of equal size, we write $\ve{x}{y} = x^\top y = \sum_{i} x_i y_i$, and by $x\circ y$ we denote their  Hadamard (i.e., elementwise) product. We also write $u^{-1}=(u^{-1}_1,\dots,u_n^{-1})$.

\textbf{Matrices.} $\bI$ is the identity matrix  in $\R^{n \times n}$ and $\bD(w)$ is the diagonal matrix in $\R^{n\times n}$ with $w\in \R^n$ on its diagonal. We will write  $\bM\succeq 0$ (resp. $\bM\succ 0$)  to indicate that  $\bM$ is positive semidefinite (resp. positive definite).

\textbf{Subsets of coordinates.} Let  $S$ be a nonempty subset of $[n]:=\{1,2,\dots,n\}$. For any matrix $\bM \in \R^{n\times n}$ we write 
$\bM_S$ for the matrix  obtained from $\bM$ by 
retaining elements $\bM_{ij}$ for which both $i\in S$ and $j\in S$ and zeroing out all other elements. Clearly, $\bM_S = \bI_S \bM \bI_S$. Moreover, for any vector $h\in \R^n$ we write 
\begin{equation}\label{eq:decomposition_of_h}
\textstyle h_S := \bI_S h = \sum_{i=1}^n h_i e_i.
\end{equation}
Note that we can thus write
\begin{equation}\label{eq:hhatS}
(h_S)^\top \bM h_{S}= h^\top \bI_S \bM \bI_S h = h^\top \bM_{S} h,
\end{equation}
and
that for $x,y\in \R^n$ we have
\begin{equation}\label{eq:hSw}
 \ve{x_S}{y}= \ve{\bI_S x}{y} = \ve{x}{\bI_S y} = \ve{x}{y_S}.
\end{equation}
By $\pseudoinv{(\bM_{S})}$  we denote the matrix in $\R^{n\times n}$ for which 
 \begin{equation}\label{eq:inverses} \pseudoinv{ (\bM_{S}) } \bM_{S} = \bM_{S} \pseudoinv{(\bM_{S}) } =  \bI_{S}.\end{equation}

\section{Minimization of a Smooth Function} \label{sec:smooth}

In this section we consider unconstrained minimization of a differentiable convex function:
\begin{equation}\label{eq:probse}
 \min_{x\in \R^n}    f(x).
 \end{equation}
In particular, we shall assume smoothness (Lipschitz continuity of the gradient) and strong convexity of $f$:
\begin{assumption}[Smoothness] \label{ass:smoothness}  There is a positive definite matrix $\bM\in \R^{n\times n}$
 such that for all $x,h\in \R^n$,
 \begin{equation}f(x+h) \leq f(x) + \ve{\nabla f(x)}{h} + \frac{1}{2}\ve{\bM h}{h} \label{a:smoothness} 
\end{equation}
\end{assumption}

\begin{assumption}[Strong convexity]\label{ass:strong_convexity}  There is a positive definite matrix $\bG\in \R^{n\times n}$
 such that for all $x,h\in \R^n$,
\begin{equation}f(x) + \ve{\nabla f(x)}{h} + \frac{1}{2}\ve{\bG h}{h} \leq f(x+h).\label{a:strong_convexity} 
\end{equation}
\end{assumption}


\subsection{Three stochastic algorithms}

We now describe three algorithmic strategies for solving problem~\eqref{eq:probse}, {\bf the first two of which are new.} All these methods  have the form
\begin{equation}\label{eq:4methods}x^{k+1} \leftarrow x^k + h^k,\end{equation}
where $h_i^k$ is only allowed to be nonzero for $i\in S_k$, where $\{S_k\}_{k\geq 0}$ are i.i.d. random subsets of $[n]:=\{1,2,\dots,n\}$ (``samplings''). That is, all methods in each iteration update a random subset of the variables. The four methods will only differ in how the update elements $h_i^k$ for $i\in S_k$ are computed. If we wish the methods to work, we necessarily need to require that every coordinate has a positive probability of being sampled. For certain technical reasons that will be apparent later, we will also assume that $S_k$ is nonempty with probability 1. 

\begin{assumption}[Samplings] \label{ass:samplings} The random sets   $\{S_k\}_{k\geq 0}$ are i.i.d., proper
 (i.e., $\Prob(i\in S_k)>0$ for all $i\in [n]$) and nonvacuous (i.e., $\Prob(S_k=\emptyset)=0$). 
\end{assumption}

Much of our discussion will depend on the distribution of $S_k$ rather than on $k$. As $\{S_k\}_{k\geq 0}$ are i.i.d., we will write $\Sam$ for a sampling which shares their distribution. We will write $p=(p_1,\dots,p_n)$ where 
\begin{equation}
\label{eq:p_i}p_i := \Prob(i\in \Sam), \qquad i\in [n].\end{equation}
By Assumption~\ref{ass:samplings}, we have $p_i>0$ for all $i$. We now describe the methods.

\textbf{Method 1.} We compute $\pseudoinv{(\bM_{S_k})}$ and set
\[\tag{Method $1$} h^k = - \pseudoinv{(\bM_{S_k})} \nabla f(x^k).\]
Note that the update only involves the inversion of a random principal submatrix of $\bM$ of size $|S_k|\times |S_k|$. Also, we only need to compute elements $i\in S_k$ of the gradient $\nabla f(x_k)$. If $|S_k|$ is reasonably small, the update step is cheap.

\textbf{Method 2.}  We compute the inverse of $\Exp[\bM_{\Sam}]$ and set
\[\tag{Method $2$} h^k = - \bI_{S_k} (\Exp [\bM_{\Sam}])^{-1} \bD(p)\nabla f(x^k).\]
This strategy easily implementable when $|\Sam|=1$ with probability 1 (i.e., if we update a single variable only). This is because then $\Exp[\bM_{\Sam}]$ is a diagonal matrix with the $(i,i)$ element equal to $p_i \bM_{ii}$. Hence, the update step simplifies to $h^k_i = - \tfrac{1}{\bM_{ii}} \ve{e_i}{\nabla f(x^k)}$ for $i\in S_k$ and $h^k_i=0$ for $i\notin S_k$. For more complicated samplings $\Sam$, however, the matrix $\Exp[\bM_{\Sam}]$ will be as hard to invert as $\bM$.

\textbf{Method 3.} We compute a vector $v\in \R^n$ for which
\begin{equation}\label{eq:ESO}\Exp[\bM_{\Sam}] \preceq \bD(p)\bD( v)\end{equation} and then set
\begin{equation}\label{eq:STRATEGY4}\tag{Method $3$} h^k = -  \bI_{S_k} (\bD( v))^{-1}  \nabla f(x^k).\end{equation}
Assuming $v$ is easily computable (this should be done before the methods starts), the update is clearly very easy to perform. Indeed,
the update can be equivalently written as $h_i^k = -\tfrac{1}{v_i} \ve{e_i}{\nabla f(x^k)}$ for $ i\in S_k$ and $h_i^k=0$ for $i\notin S_k$. Method 3 is known as  NSync  \cite{NSync}. For a calculus allowing the computation of closed form formulas for $v$ as a function of the sampling $\Sam$ we refer the reader to \cite{ESO}.

Note that {\bf all three methods coincide} if $|\Sam|=1$ with probability 1.

\subsection{Three linear convergence rates}

We shall now show that, putting the issue of the cost of each iteration of the three methods  aside, all enjoy a linear rate of convergence. In particular, we shall show that Method 1 has the fastest rate, followed by Method $2$ and finally, Method $3$.

\begin{theorem}\label{thm:3} Let Assumptions~\ref{ass:smoothness}, \ref{ass:strong_convexity} and \ref{ass:samplings} be satisfied.  Let $\{x^k\}_{k\geq 0}$ be the sequence of random vectors produced by Method $m$, for $m=1,2,3$ and let $x^*$ be the optimal solution of~\eqref{eq:probse}. Then
\[\Exp [f(x^{k+1})-f(x^*)] \leq (1-\sigma_m)\Exp[f(x^k)-f(x^*)],\]
where
\begin{align}
\sigma_1 &:=  \lambda_{\min}\left(\bG^{1/2}\Exp\left[ \pseudoinv{\left(\bM_{\Sam}\right)}\right]\bG^{1/2}\right), \label{eq:sigma1}\\
\sigma_2 &:=  \lambda_{\min}\left(\bG^{1/2}\bD(p)\left(\Exp\left[\bM_{\Sam}\right]\right)^{-1} \bD(p)\bG^{1/2}\right),\label{eq:sigma2}\\
\sigma_3 &:= \lambda_{\min}\left(\bG^{1/2}\bD(p)\bD(v^{-1})\bG^{1/2}\right).\label{eq:sigma3}
\end{align}

\end{theorem}

The above result means that the number of iterations sufficient for Method $m$ to obtain an $\epsilon$-solution (in expectation) is $O(\tfrac{1}{\sigma_m}\log(1/\epsilon))$. 

In the above theorem (which we prove in Section~\ref{sec:proof_3_methods}), $\lambda_{min}(\bX)$ refers to the smallest eigenvalue of matrix $\bX$. It turns out that in all three cases, the matrix $\bX$ involved is positive definite. However, for the matrices in  \eqref{eq:sigma1} and \eqref{eq:sigma2} this will only be apparent if we show that $\Exp[\bM_{\Sam}]\succ 0$ and $\Exp[\pseudoinv{(\bM_{\Sam})}] \succ 0$, which we shall do next.

\begin{lemma}\label{lem:1}
If $\Sam$ is a proper sampling, then
$\Exp\left[\bM_{\Sam}\right]\succ 0$.
\end{lemma}
\begin{proof}
Denote $\support \{x\}=\{i\in [n]: x_i\neq 0\}$.
Since $\bM\succ 0$, any principal submatrix of $\bM$ is also positive definite. Hence  for any $x\in \R^n \backslash \{0\}$, 
$x^\top \bM_{S} x=0$ implies that $\support\{x\} \cap S=\emptyset$ for all $S\subseteq [n]$.  If $x\in \R^n $ is such that
$$
\textstyle x^\top \Exp\left[\bM_{\Sam}\right] x= \sum_{S\subseteq [n]} \Prob(\Sam=S) x^\top \bM_{S} x=0,
$$
then $\Prob(\support\{x\} \cap \Sam=\emptyset)=1$. Since $\Sam$ is proper, this only happens 
when $x=0$. Therefore, $\Exp[\bM_{\Sam}]\succ 0$.
\end{proof}

\begin{lemma}\label{l-dzeff} If $\Sam$ is proper and nonvacuous,
then 
 \begin{align}\label{a-dzeff}
0\prec \bD(p) \pseudoinv{\left(\Exp \left[ \bM_{\Sam}\right]\right) } \bD(p)  \preceq 
\Exp\left[\pseudoinv{\left(\bM_{\Sam} \right) }\right]
.\end{align}
\end{lemma}
\begin{proof}
The first inequality follows from Lemma~\ref{lem:1} and the fact for proper $\Sam$ we have $p>0$ and hence $\bD(p)\succ 0$.
We now turn to the second inequality. Fix $h\in\R^n$. For arbitrary $\emptyset\neq  S\subseteq [n]$ and $y\in \R^n$ we have:
\begin{gather*}
\tfrac{1}{2}h^\top \pseudoinv{\left(\bM_{S} \right) } h =\tfrac{1}{2}h_{S}^\top  \pseudoinv{\left(\bM_{S} \right)} h_{S}\\ 
= \max_{x\in \R^n}
\ve{x}{h_{S}}-\tfrac{1}{2} x^\top \bM_{S} x \geq  \ve{y}{h_{S}}-\tfrac{1}{2} y^\top \bM_{S} y.
\end{gather*}
Substituting $S=\Sam$ and taking expectations, we obtain
\begin{gather*}
\tfrac{1}{2}\Exp\left[h^\top \pseudoinv{\left(\bM_{\Sam} \right)} h\right] \geq
\Exp\left[\ve{y}{h_{\Sam}}-\tfrac{1}{2} y^\top \bM_{\Sam} y \right]
\\
=y^\top \bD(p) h-\tfrac{1}{2}y^\top \Exp\left[\bM_{\Sam}\right] y.
\end{gather*}
Therefore, $
 \tfrac{1}{2}h^\top\Exp\left[ \pseudoinv{\left(\bM_{\Sam} \right)} \right] h \geq
\max_{y\in  \R^n } y^\top \bD(p) h -\tfrac{1}{2}y^\top \Exp\left[\bM_{\Sam}\right] y 
= \tfrac{1}{2} h^\top \bD(p) \pseudoinv{\left(\Exp \left[ \bM_{\Sam}\right]\right)}\bD(p) h.$\qedhere
\end{proof}

We now establish an important relationship between the quantities $\sigma_1,\sigma_2$ and $\sigma_3$, which sheds light on the convergence rates of the three methods.

\begin{theorem}\label{thm:3_rates} $0<\sigma_3 \leq \sigma_2 \leq  \sigma_1\leq 1$.
\end{theorem}
\begin{proof} We have $\sigma_m>0$ for all $m$ since $\sigma_m$ is the smallest eigenvalue of a positive definite matrix. That $\sigma_m\leq 1$ follows as a direct corollary Theorem~\ref{thm:3}. Finally,
$\bD(p)\bD(v^{-1}) =
\bD(p)\bD(p^{-1})\bD(v^{-1}) \bD(p)  \overset{\eqref{eq:ESO}}{\preceq} \bD(p)\left(\Exp \left[ \bM_{\Sam}\right]\right)^{-1}\bD(p)\overset{ \eqref{a-dzeff}}{\preceq} \Exp\left[ \pseudoinv{\left(\bM_{\Sam} \right)}\right].
$
\end{proof}

\subsection{Example}

Consider the function $f: \R^3\to \R$ given by
$$f(x)=\tfrac{1}{2}x^T \bM x, \quad \scriptsize
\bM=\begin{pmatrix}
 1.0000 & 0.9900 & 0.9999 \\
 0.9900 & 1.0000 & 0.9900 \\
 0.9999 & 0.9900 & 1.0000
                   \end{pmatrix}.
$$
Note that Assumption~\ref{ass:smoothness} holds, and Assumption~\ref{ass:strong_convexity} holds with $\bG=\bM$.
Let $\Sam$ be the ``$2$-nice sampling'' on $[n]=\{1,2,3\}$. That is, we set
$
\Prob(\Sam=\{i,j\})=\frac{1}{3}.
$ for $(i,j) = (1,2), (2,3), (3,1)$.
A straightforward calculation reveals that:
\[\scriptsize
\Exp\left[ \pseudoinv{\left(\bM_{\Sam} \right) }\right]\approx\begin{pmatrix}
1683.50  & -16.58 & -1666.58 \\
  -16.58 &  33.50 &  -16.58 \\
  -1666.58 &  -16.58 &   1683.50
                   \end{pmatrix},
\]
\[
\scriptsize \bD(p)\left(\Exp \left[ \bM_{\Sam}\right]\right)^{-1}\bD(p) \approx \begin{pmatrix}
 0.9967 &  -0.3268  & -0.3365 \\
   -0.3268  &  0.9902   & -0.3268\\
   -0.3365 &  -0.3268 &   0.9967
                   \end{pmatrix}.
\]

It can be verified that \eqref{eq:ESO} holds with $v=(2,2,2)$; see \cite{PCDM} or \cite{ESO}. Therefore, $\bD(p)\bD(v^{-1})=\tfrac{1}{3}\bI$. Finally, we obtain:
\[\sigma_1 \approx 0.3350, \; \sigma_2 \approx 1.333\cdot 10^{-4}, \;
 \sigma_2 \approx  0.333\cdot 10^{-4}.
\]

Note that:  theoretical rate, $\sigma_1$, of Method 1 is {\bf 10,000 times better} than the rate, $\sigma_3$, of parallel coordinate descent (Method 3).


\subsection{Proof of Theorem~\ref{thm:3}}\label{sec:proof_3_methods}
\begin{proof}
By minimizing both sides of  \eqref{a:strong_convexity}  in $h$, we get:
\begin{align}\label{a-ff*}
f(x)-f(x^*)\leq \frac{1}{2}\<\nabla f(x), \bG^{-1}\nabla f(x)>.
\end{align}
In view of  \eqref{a:smoothness} and \eqref{eq:hhatS}, for  for all $h\in\R^n$ we have:
\begin{equation}\label{eq:09809}
f(x^k+\bI_{S_k} h) \leq f(x^k)+\ve{\nabla f(x^k)}{\bI_{S_k} h} +\frac{1}{2}\ve{\bM_{S_k} h}{ h}.
\end{equation}
{\em Method 1:} If we use \eqref{eq:09809} with $h\leftarrow h^k := - \pseudoinv{(\bM_{S_k})}\nabla f(x^k)$, and apply \eqref{eq:inverses}, we get:
\[f(x^{k+1})\leq f(x^k)-\frac{1}{2}\<\nabla f(x^k),\pseudoinv{(\bM_{S_k})} \nabla f(x^k)>.\]
Taking expectations on both sides with respect to $S_k$ yields:
\begin{eqnarray*}
& &\Exp_k[f(x^{k+1})]\\&\leq & f(x^k)-\frac{1}{2}\ve{\nabla f(x^k)}{\Exp [\pseudoinv{\left(\bM_{\Sam}\right)}]\nabla f(x^k)}\\
& \overset{\eqref{eq:sigma1}}{\leq}  & f(x^k)-\frac{\sigma_1}{2}\<\nabla f(x^k),\bG^{-1}\nabla f(x^k)>\\
& \overset{\eqref{a-ff*}}{\leq} & f(x^k)-\sigma_1 \left(f(x^k)-f(x^*)\right),
\end{eqnarray*}
where $\Exp_k$ denotes the expectation with respect to $S_k$.
It remains to rearrange the inequality and take expectation.

{\em Method 2:} Let $\bD=\bD(p)$. Taking expectations on both sides of \eqref{eq:09809} with respect to $S_k$, we see that for all $h\in \R^n$ the following holds:
$
\Exp_k [f(x^k+\bI_{S_k} h)] \leq f(x^k)+\ve{\bD \nabla f(x^k)}{ h} 
 +\frac{1}{2}\ve{\Exp[\bM_{S_k}] h}{ h}.
$ Note that the choice $\tilde{h}^k:= - (\Exp [\bM_{\Sam}])^{-1} \bD \nabla f(x^k)$ minimizes the right hand side of the inequality in $h$. Since $h^k = \bI_{S_k} \tilde{h}^k$, 
\begin{eqnarray*}
&& \Exp_k [ f(x^{k+1}) ] \\
&\leq& f(x^k) - \frac{1}{2} \ve{\nabla f(x^k)}{\bD \left(\Exp[\bM_{\Sam}]\right)^{-1} \bD \nabla f(x^k)}\\
& \overset{\eqref{eq:sigma2}}{\leq}  & f(x^k)-\frac{\sigma_2}{2}\<\nabla f(x^k),\bG^{-1}\nabla f(x^k)>\\
& \overset{\eqref{a-ff*}}{\leq} & f(x^k)-\sigma_2 \left(f(x^k)-f(x^*)\right).
\end{eqnarray*}

{\em Method 3:} The proof is the same as that for Method 2, except in the first inequality we replace $\Exp[\bM_{S_k}]$ by $\bD(p)\bD(v)$ (see \eqref{eq:ESO}).
\end{proof}

\section{Minimization of a  Composite Function} \label{sec:proximal}

In this section we consider the following {\em composite} minimization problem:
\begin{equation}\label{eq:probse1}
 \min_{x\in \R^n}   F(x)\equiv f(x)+\sum_{i=1}^n \psi_i(x_i).
\end{equation}
We assume that $f$ satisfies Assumptions~\ref{a:smoothness} and  \ref{a:strong_convexity}. The difference from the setup in the previous section is in the inclusion of the separable term $\sum_i \psi_i$. 

\begin{assumption} \label{ass:composite} For each $i$,  $\psi_i:\R\rightarrow \R \cup\{+\infty\}$ is closed and $\gamma_i$-strongly convex for some $\gamma_i \geq 0$. Let $\gamma=(\gamma_1,\dots,\gamma_n)\in \R_+^n$. 
\end{assumption}

For ease of presentation, in this section we only consider  \textit{uniform sampling} $\Sam$, which means that $\Prob(i\in \Sam) = \Prob(j\in \Sam)$ for all $i,j\in [n]$. In particular, this implies that
$\Prob(i\in \Sam)=\frac{\Exp[|\Sam|]}{n}$ for all $i$. Let $\tau:=\Exp[\Sam]$.

\subsection{New algorithm}

We now propose  Algorithm~\ref{alg:a1}, which a variant of Method 1 applicable to problem~\eqref{eq:probse1}. If $\psi_i\equiv 0$ for all $i$, the methods coincide. The following result states that the method converges at a geometric rate, in expectation.

\begin{algorithm}[ht!]
\caption{Proximal version of Method 1}
\label{alg:a1}
\begin{algorithmic}[1]
\STATE \textbf{Parameters:}  uniform sampling $\Sam$
 \STATE \textbf{Initialization:} choose initial point $x^0\in\R^n$
\FOR {$k=0,1,2,\dots$}
 \STATE Generate a random set of blocks $S_k\sim \Sam $
 \STATE Compute: $h^k=\argmin_{h\in \R^n} \ve{\nabla f(x^k)}{h_{S_k}}+\frac{1}{2}\ve{h}{\bM_{S_k} h} +\sum_{i\in S_k}\psi_i(x_i^k+h_i)
$
 \STATE Update: $x^{k+1}:= x^k+ h^k_{S_k}$
\ENDFOR
\end{algorithmic}
\end{algorithm}

\begin{theorem}\label{prop-proximal} Let Assumptions~\ref{ass:smoothness}, \ref{ass:strong_convexity}, \ref{ass:samplings} and \ref{ass:composite} be satisfied. Then the output sequence $\{x^k\}_{k\geq 0}$ of 
Algorithm~\ref{alg:a1} satisfies:
$$
\Exp[F(x^{k+1})-F(x^*)]\leq (1-\sigma^{prox}_1)\Exp[F(x^k)-F(x^*)],
$$
where $x^*$ is the solution of~\eqref{eq:probse1}, $\sigma^{prox}_1:=\frac{\tau \min(1,s_1)}{n}$ and
$$
s_1:=\lambda_{\min} \left[ \left(\frac{n}{\tau}\Exp[ \bM_{\Sam}]+ \bD(\gamma)\right)^{-1}(\bD(\gamma)+\bG) \right].
$$
\end{theorem}
Note for positive definite matrices $\bX,\bY$, we have $\lambda_{\min}(\bX^{-1}\bY) = \lambda_{\min}(\bY^{1/2}\bX^{-1}\bY^{1/2})$. It is this latter form we have used in the formulation of Theorem~\ref{thm:3}. In the special case when $\gamma\equiv 0$ ($\psi_i$ are merely convex), we have $\sigma^{prox}_1 =\min\{\tfrac{\tau}{n},\tfrac{\tau^2}{n^2}\lambda_{\min}(\bG^{1/2}(\Exp[\bM_{\Sam}])^{-1}\bG^{1/2})\}$. Note that while this rate applies to a proximal/composite variant of Method 1, its rate is best compared to the rate $\sigma_2$ of Method 2. Indeed,  looking at \eqref{eq:sigma2}, and realizing that for uniform $\Sam$ we have $\bD(p)=\tfrac{\tau}{n}\bI$, we get \[\sigma_1\geq \sigma_2 = \tfrac{\tau^2}{n^2}\lambda_{\min}(\bG^{1/2}(\Exp[\bM_{\Sam}])^{-1}\bG^{1/2})\geq \sigma_1^{prox}.\] 
So, the rate we can prove for the composite version of Method 1 ($\sigma^{prox}_1$) is weaker than the rate we get for Method 2 ($\sigma_2$), which by Theorem~\ref{thm:3_rates} is weaker than the rate  of Method 1 ($\sigma_1$). We believe this is a byproduct of our analysis  rather than the weakness of Algorithm~\ref{alg:a1}.

\subsection{PCDM}

We will now compare our new Algorithm~\ref{alg:a1} with the Parallel Coordinate Descent Method (PCDM) of \citet{PCDM}, which can also be applied to problem \eqref{eq:probse1}.

\begin{algorithm}[ht!]
\caption{PCDM~\cite{PCDM}}
\label{alg:pcdm}
\begin{algorithmic}[1]
\STATE \textbf{Parameters:}  uniform sampling $\Sam$; $v\in \R^n_{++}$
 \STATE \textbf{Initialization:} choose initial point $x^0\in\R^n$
\FOR {$k=0,1,2,\dots$}
 \STATE Generate a random set of blocks $S_k\sim \Sam $
 \STATE Compute for $i \in S_k$
 
 $
h_i^k=\displaystyle\argmin_{h_i\in \R} e_i^\top \nabla f(x^k) h_i +\frac{v_i}{2}|h_i|^2+\psi_i(x_i^k+h_i)
$
 \STATE Update: $x^{k+1}:= x^k+ \sum_{i\in S_k} h_i^k e_i$
\ENDFOR
\end{algorithmic}
\end{algorithm}


\begin{proposition} Let the same assumptions as those in Theorem \ref{prop-proximal} be satisfied. Moreover, assume $v\in \R_{++}^n$ is a vector satisfying \eqref{eq:ESO}.
Then the output sequence $\{x^k\}_{k\geq 0}$ of Algorithm~\ref{alg:pcdm} satisfies
$$
\Exp[F(x^{k+1})-F(x^*)]\leq (1- \sigma^{prox}_3)\Exp[F(x^k)-F(x^*)],
$$ 
where $\sigma_3^{prox}:=\frac{\tau \min(1,s_3)}{n}$ and
$$
 s_3:=\lambda_{\min} \left[ \left(\bD(v+\gamma)\right)^{-1}(\bD(\gamma)+\bG) \right].
$$
\end{proposition}
\begin{proof} {\em Sketch:} The proof is a minor modification of the arguments in \cite{PCDM}.
\end{proof}

\subsection{Comparison of the rates of Algorithms~\ref{alg:a1} and \ref{alg:pcdm}}

We now show that the rate of linear (geometric) convergence of our method is better than that of PCDM. 

\begin{proposition}  $\sigma^{prox}_{1}\geq \sigma^{prox}_{3}$.
\end{proposition}
\begin{proof}
Since $p_i=\tfrac{\tau}{n}$ for all $i$, we have $\bD(p)=\tfrac{\tau}{n}\bI$ and hence from~\eqref{eq:ESO} we deduce that:
\[ \frac{n}{\tau}\Exp[ \bM_{\Sam}]+ \bD(\gamma) \overset{\eqref{eq:ESO}}{\preceq} \bD(v)+ \bD(\gamma) = \bD(v+\gamma),\]
whence  $s_1 \geq s_3$, and the claim follows.
\end{proof}

\section{Empirical Risk Minimization} \label{sec:ERM}

We now turn our attention to the \textit{empirical risk minimization} problem:
\begin{equation}\label{eq:primal}
\min_{w \in \R^d}
P(w) := \tfrac{1}{n}\sum_{i=1}^n \phi_i(a_i^\top w) + \lambda g(w).\end{equation}
We assume that   $g:\R^d\rightarrow \R$ is a 1-strongly convex function with respect 
to the L2 norm  and each loss function $\phi_i:\R\rightarrow \R$ is convex and $1/\gamma$-smooth. Each $a_i $ is a $d$-dimensional vector and for ease of presentation we write $\bA=(a_1,\dots,a_n)=\sum_{i=1}^ n  a_i e_i^\top $. Let $g^*$ and $\{\phi_i^*\}_i$ be  the Fenchel conjugate functions of $g$ and $\{\phi_i\}_i$, respectively. In the case of $g$, for instance, we have  $g^*(s) = \sup_{w\in \R^d} \ve{w}{s}-g(w)$. The (Fenchel) dual problem of~\eqref{eq:primal} can be written as:
\begin{equation}\label{eq:dual}\max_{\alpha \in \R^{n}} 
D(\alpha)
:= \tfrac{1}{n}\sum_{i=1}^n -\phi_i^*(-\alpha_i)-\lambda g^*\left(\tfrac{1}{\lambda n} \bA \alpha \right).\end{equation}

\subsection{SDNA: A new algorithm for ERM}

Note that the dual problem has the form \eqref{eq:probse1}
\begin{equation}\label{eq:Dual-Alternative} \min_{\alpha\in \R^n} F(\alpha) \equiv  f(\alpha) + \sum_{i=1}^n \psi_i(\alpha_i),\end{equation} 
where  $F(\alpha)=-D(\alpha)$, $f(\alpha) = \lambda g^*(\tfrac{1}{\lambda n}\bA \alpha)$ and $\psi(\alpha_i) = \tfrac{1}{n}\phi_i^*(-\alpha_i)$. It is easy to see that $f$ satisfies Assumption~\ref{ass:smoothness} with $\bM:= \tfrac{1}{n}\bX$, where $\bX:=\frac{1}{\lambda n}\bA^\top \bA$. Moreover, $\psi_i$ is $\tfrac{\gamma}{n}$-strongly convex. We can therefore apply Algorithm~\ref{alg:a1} to solve the dual \eqref{eq:Dual-Alternative}. This is what Algorithm~\ref{algorithm:PRCDM-composite} does.

\begin{algorithm}[ht!]
\caption{Stochastic Dual Newton Ascent (SDNA)}
\label{alg:pd}
\begin{algorithmic}[1] \label{algorithm:PRCDM-composite}
\STATE \textbf{Parameters:}  proper  nonvacuous sampling $\Sam$
 \STATE \textbf{Initialization:} $\alpha^0\in\R^n$; $\bar\alpha^0=\frac{1}{\lambda n} \bA \alpha^0$
\FOR {$k=0,1,2,\dots$}
\STATE Primal update: $w^k= \nabla g^*(\bar \alpha^k)$
 \STATE Generate a random set of blocks $S_k\sim \Sam $
 \STATE Compute: \vspace{-0.2in}
$$
\Delta \alpha^k=\argmin_{h\in \R^n} \ve{(\bA^\top w^k)_{S_k}}{ h}+\tfrac{1}{2} h^\top \bX_{S_k} h$$\vspace{-0.2in}\\$
\qquad \qquad\qquad\qquad+\sum_{i\in S_k}\phi_i^*(-\alpha_i^k-h_i)
$
 \STATE Dual update: $\alpha^{k+1}:= \alpha^k+ (\Delta \alpha^k)_{S_k}$
\STATE Average update: $\bar \alpha^{k+1}=\bar \alpha^k+\frac{1}{\lambda n} \sum_{i\in S_k}\Delta\alpha_i^k a_i $
\ENDFOR
\end{algorithmic}
\end{algorithm}

If  $\alpha^*$ is the optimal solution of~\eqref{eq:dual}, then the optimal solution of~\eqref{eq:primal} is given by:
\begin{align}\label{a:primaldualsolution}
w^*=\nabla g^* \left(\tfrac{1}{\lambda n} \bA \alpha^*\right).
\end{align}

With each proper sampling $\Sam$ we associate the number:
\begin{align}\label{a:theta}
 \theta (\Sam):= \min_i \frac{p_i  \lambda \gamma n}{v_i+\lambda \gamma n},
\end{align}
where
$(p_1,\dots,p_n)$ is the vector of probabilities defined in~\eqref{eq:p_i} and $v=(v_1,\dots,v_n)\in \R_{++}^n$ is a vector satisfying:
\begin{align}\label{a:PM}
\Exp[(\bA^\top \bA)_{\Sam}]\preceq \bD( p)\bD( v).
\end{align}
Closed-form expressions for $v$ satisfying this inequality, as a function of the sampling $\Sam$ chosen, can be found in \cite{ESO}. A rather conservative choice which works for any $\Sam$, irrespective of its distribution, is $v_i = \min\{\tau,\lambda'(\bA^\top\bA)\}\|a_i\|^2$, where $\lambda'(\bY):=\max_h \{h^\top \bY h \;:\; h^\top \bD(\bY)h\leq 1\}$ and $\tau$ is a number for which $|\Sam |\leq \tau$ with probability 1 (see Theorem 5.1 in the aforementioned reference). Better bounds (with smaller $v$) can be derived for special classes of samplings.

Now we can state the main result of this section:
\begin{theorem}[Complexity of SDNA]\label{th-dual-proximal}
Let $\Sam$ be a uniform sampling and let  $\tau:=\Exp[|\Sam|]$. The output sequence $\{w^k,\alpha^k\}_{k\geq 0}$ of Algorithm~\ref{alg:pd} satisfies:
$$
\Exp[P(w^k)-D(\alpha^k)]\leq \frac{\left(1- \sigma_1^{prox}\right)^k}{\theta(\Sam)}  (D(\alpha^*)-D(\alpha^0)),
$$
where $\sigma_1^{prox}:=\frac{\tau \min(1,s_1)}{n}$ and
\begin{equation}\label{eq:s_2}s_1=\lambda_{\min} \left[ \left(\frac{1}{\tau\gamma \lambda}\Exp[(\bA^\top\bA)_{\Sam}]+ \bI\right)^{-1} \right].\end{equation}
\end{theorem}

In the case of quadratic losses and quadratic regularizer, we can sharpen the complexity bound:
\begin{theorem}\label{th-dual-smooth}
 When both $\phi_i$ and $g$ are quadratic functions, the output sequence $\{w^k,\alpha^k\}_{k\geq 0}$ 
of Algorithm~\ref{alg:pd} satisfies:
$$
\Exp[P(w^k)-D(\alpha^k)]\leq \frac{(1-\sigma_1)^k}{\theta(\Sam)}  \left(D(\alpha^*)-D(\alpha^0)\right)
$$
where
$$
\sigma_1:=\lambda_{\min}\left[ \Exp\left[\left(\left(\tfrac{1}{\lambda n}\bA^\top \bA+\gamma\bI\right)_{\Sam}\right)^{-1}
\left(\tfrac{1}{\lambda n}\bA^\top \bA+\gamma\bI\right) \right]\right].
$$

\end{theorem}

\subsection{Complexity analysis}

We first establish that SDNA is able to solve the dual.
\begin{lemma}\label{lem:mainxxx}
Let $\Sam$ be a uniform sampling and $\tau:=\Exp[|\Sam|]$. The output sequence $\{\alpha^k\}_{k\geq 0}$ of Algorithm~\ref{alg:pd}
satisfies:
$$
\Exp[D(\alpha^*)-D(\alpha^k)] \leq \left(1-\sigma_1^{prox}\right)^k (D(\alpha^*)-D(\alpha^0)),
$$
where $\sigma_1^{prox}$ is as in Theorem~\ref{th-dual-proximal}.
\end{lemma}
\begin{proof}
If $\Sam$ is uniform,
then the output of Algorithm~\ref{alg:pd} is equivalent to the output of 
Algorithm~\ref{alg:a1} applied to \eqref{eq:Dual-Alternative}. Therefore, the result is obtained by
 applying Theorem~\ref{prop-proximal}
 with $\bM=\tfrac{1}{\lambda n^2}\bA^\top \bA$, $\bG=0$ and $\gamma_i=\tfrac{\gamma}{n}$ for all $i$.
\end{proof}

We now prove a sharper result in the case of quadratic loss and quadratic regularizer.

\begin{lemma}\label{lem:98h98s}
 If  $\{\phi_i\}_i$ and $g$ are quadratic, then the output sequence $\{\alpha^k\}_{k\geq 0}$ of Algorithm~\ref{alg:pd}
satisfies:
$$
\Exp[D(\alpha^*)-D(\alpha^k)] \leq (1-\sigma_1)^k (D(\alpha^*)-D(\alpha^0)),
$$
where  $\sigma_1$ is as in Theorem~\ref{th-dual-smooth}.
\end{lemma}
\begin{proof}
 If  $\{\phi_i\}_i$ and $g$ are all quadratic functions, then the dual objective function is quadratic with Hessian matrix given by $
\nabla^2 D(\alpha)\equiv  \frac{1}{\lambda n^2}\bA^\top \bA+\frac{\gamma}{n}\bI$. It suffices to apply Theorem~\ref{thm:3}\eqref{eq:sigma1}, with $\bM=\bG=\nabla^2 D(\alpha)$.
\end{proof}

We now prove a more general version of a classical result in dual coordinate ascent methods which bounds the duality gap from above by the expected  dual increase.
\begin{lemma}\label{l-pdgd}
The output sequence $\{w^k,\alpha^k\}_{k\geq 0}$ of Algorithm~\ref{alg:pd} satisfies:
\[
 \Exp_k[D(\alpha^{k+1})-D(\alpha^k)] \geq \theta(\Sam) (P(w^k)-D(\alpha^k)).
\]
\end{lemma}
 
 The proof of the this lemma is provided in the supplementary material. Theorem~\ref{th-dual-proximal} (resp. Theorem~\ref{th-dual-smooth}) now follows by combining Lemma~\ref{lem:mainxxx} (resp. Lemma~\ref{lem:98h98s}) and  Lemma~\ref{l-pdgd}.

\subsection{New Algorithm: SDCA with Arbitrary Sampling}

When $|\Sam |=1$ with probability 1, 
SDNA reduces to a proximal variant of stochastic 
dual coordinate ascent (SDCA)~\cite{SDCA}. However, a {\bf minibatch version of standard SDCA in the ERM setup we consider here has not been previously studied in the literature.}   \citet{Pegasos2}  developed such a method but in the special case of hinge-loss (which is not smooth and hence does not fit our setup). \citet{ASDCA} studied minibatching but in conjunction with acceleration and the QUARTZ method of 
\citet{QUARTZ}, which has been analyzed for an arbitrary sampling $\Sam$, uses a different primal update than SDNA. Hence, in order to compare SDNA with an SDCA-like method which is as close a match to SDNA as possible, we need to develop a new method. {\bf Algorithm~\ref{alg:sdca} is an extension of SDCA to allow it handle an arbitrary uniform sampling $\Sam$.}


The complexity of Minibatch SDCA (we henceforth just write SDCA) is given in Theorem~\ref{th-sdca}.

\begin{theorem}\label{th-sdca}
If~\eqref{a:PM} holds, then the output sequence $\{w^k, \alpha^k\}_{k\geq 0}$ of Algorithm~\ref{alg:sdca} satisfies:
$$
\Exp[P(w^k)-D(\alpha^k)]\leq \frac{(1- \theta(\Sam))^k}{\theta(\Sam)}  \left(D(\alpha^*)-D(\alpha^0)\right).
$$
\end{theorem}

\begin{algorithm}[ht!]
\caption{Minibatch SDCA}
\label{alg:sdca}
\begin{algorithmic}[1] 
\STATE \textbf{Parameters:}  uniform sampling $\Sam$,  vector $v\in \R_{++}^n$
 \STATE \textbf{Initialization:} $\alpha^0\in\R^n$;  set $\bar\alpha^0=\frac{1}{\lambda n} \bA \alpha^0$
\FOR {$k=0,1,2,\dots$}
\STATE Primal update: $w^k= \nabla g^*(\bar \alpha^k)$
 \STATE Generate a random set of blocks $S_k\sim \Sam $
 \STATE Compute for each $i\in S_k$ \\
$
h_i^k=\displaystyle\argmin_{h_i\in \R} h_i(a_i^\top w^k)+\frac{v_i}{2} |h_i|^2+\phi_i^*(-\alpha_i^k-h_i)
$
 \STATE Dual update: $\alpha^{k+1}:= \alpha^k+\sum_{i\in S_k} h_i^k e_i$
\STATE Average update: $\bar \alpha^{k+1}=\bar \alpha^k+\frac{1}{\lambda n} \sum_{i\in S_k}h_i^k a_i $
\ENDFOR
\end{algorithmic}
\end{algorithm}

\subsection{SDNA vs SDCA}

We now compare the rates of SDNA and  SDCA. The next result says that {\bf the rate of SDNA is always superior to that of SDCA.} We also see that the rate is better in the quadratic case covered by Theorem~\ref{th-dual-smooth}.

\begin{theorem}
If $\Sam$ is uniform sampling with $\tau=\Exp[|\Sam|]$, then
$$
\theta(\Sam)\leq \sigma_1^{prox} \leq \sigma_1.
$$
\end{theorem}
\begin{proof} Since
$\Sam$ is a uniform sampling, we have $p_i=\frac{\tau}{n}$  for all $i\in [n]$. In view of \eqref{a:theta}, we have $1\leq \tfrac{n}{\tau}\theta(\Sam)$. Next,
\begin{gather*}
s_1 \overset{\eqref{eq:s_2}+\eqref{a:PM}}{\geq} \lambda_{\min}\left( \frac{1}{\tau \lambda \gamma} \bD(v)\bD( p)+\bI\right)^{-1}  \overset{\eqref{a:theta}}{=}  \frac{n}{\tau}\theta(\Sam).
\end{gather*}
Therefore, $\sigma_1^{prox} = \tfrac{\tau}{n}\min(1,s_1) \geq \theta(\Sam)$. In order to establish $\sigma_1^{prox}\leq \sigma_1$, we use Lemma~\ref{l-dzeff} and the fact that $\Exp[\bI_{\Sam}]=\tfrac{\tau}{n}\bI$ to obtain
\begin{gather*}
\tfrac{\tau}{n}\left(\tfrac{1}{\tau\gamma \lambda }\Exp[(\bA^\top \bA)_{\Sam}]+ \bI\right)^{-1}
=\tfrac{\tau^2}{n^2} \left( \Exp\left[\left(\tfrac{1}{\gamma \lambda n}\bA^\top \bA+\bI\right)_{\Sam}\right]\right)^{-1}\\ \overset{(\text{Lemma}~\ref{l-dzeff})}{\preceq} \Exp\left[ \left(\left(\tfrac{1}{\gamma \lambda n}\bA^\top \bA+\bI\right)_{\Sam}\right)^{-1}\right]\\
 \preceq \Exp\left[\left(\left(\bA^\top \bA+\gamma\lambda n \bI\right)_{\Sam}\right)^{-1}
(\bA^\top \bA+\gamma\lambda n\bI)\right],
\end{gather*}
The rest of the argument is similar.
\end{proof}

\section{SDNA as Iterative Hessian Sketch}
We now apply SDNA to the least squares problem:
\begin{equation}\label{eq:least square}
\min_{w \in \R^d}\;\; 
  \frac{1}{2n} \|\bA^\top w-b\|^2+ \frac{\lambda}{2} \|w\|^2,
 \end{equation}
 and show that the resulting primal update can be interpreted as an iterative Hessian sketch, alternative to the one proposed by~\citet{IteHeSke}. We first need to establish a simple  duality result.
 \begin{lemma}\label{l-reporfff}
 Let $\alpha^*$ be the optimal solution of
 \begin{equation}\label{eq:dualleastsquare}\min_{\alpha \in \R^{n}}\;\; 
 \frac{1}{2n}\|\alpha\|^2 -\frac{1}{n}\ve{b}{\alpha}+\frac{1}{2\lambda n^2} \|\bA \alpha\|^2,\end{equation}
 then  the optimal solution $w^*$ of~\eqref{eq:least square}   is  $
 w^*=\frac{1}{\lambda n} \bA \alpha^*.$
 \end{lemma}
 \begin{proof}
 Problem~\eqref{eq:least square} is a special case of~\eqref{eq:primal} for
 $
  g(w)\equiv \frac{1}{2}\|w\|^2$ and $\phi_i(a)\equiv \frac{1}{2}(a-b_i)^2$ for all $ i\in [n]$. Problem~\eqref{eq:dualleastsquare} is the dual of~\eqref{eq:least square} and the result follows from~\eqref{a:primaldualsolution}.
 \end{proof}

The interpretation of SDNA as a variant of the Iterative Hessian sketch method of \citet{IteHeSke} follows immediately from the following theorem.
\begin{theorem} 
The output sequence $\{w^k,\alpha^k\}_{k\geq 0}$ of Algorithm~\ref{alg:pd} applied on problem~\eqref{eq:least square} satisfies:
\begin{align}\notag
w^{k+1}=&\argmin_{w\in \R^d}
\{ \frac{1}{2n} \|\bS_k^\top( \bA^\top w- b) \|^2+ \frac{\lambda }{2} \|w\|^2 \\ &\qquad \qquad+\ve{w}{\frac{1}{n}\bA \bI_{S_k} \alpha^k -\lambda w^k} \},\label{a-dfsdf}
\end{align}
where $\bS_k$ denotes the $n$-by-$|S_k|$ submatrix of the identity matrix $\bI_n$ with columns in the random subset $S_k$.
\end{theorem}
\begin{proof}
We know that $\bS_k^\top\Delta \alpha^k $ is the optimal solution of
 \begin{equation*}\min_{h \in \R^{\tau}}
 \frac{1}{2}\|h\|^2+\ve{\bS_k^\top(\bA^\top w^k+\alpha^k-b)}{ h}+\frac{1}{2\lambda n} \|\bA \bS_k h\|^2
 \end{equation*} Let $\tau=|S_k|$.
By Lemma~\ref{l-reporfff},  the optimal solution of
 \begin{equation*}
\min_{w \in \R^d} 
  \frac{1}{2|S_k|} \|\bS_k^\top\bA^\top w+\bS_k^\top (\bA^\top w^k+\alpha^k-b)\|^2+ \frac{\lambda n}{2|S_k|} \|w\|^2,
 \end{equation*}
 is given by $\frac{1}{\lambda n} \bA \bS_k \bS_k^\top \nabla \alpha^k$,
 which equals $\bar \alpha^{k+1}-\bar \alpha^k$ and thus equals $w^{k+1}-w^k$. Hence,
\[
w^{k+1}=\argmin_{w\in \R^d}
\{ \tfrac{1}{2n} \|\bS_k^\top( \bA^\top w+  \alpha^k- b) \|^2 + \tfrac{\lambda }{2} \|w-w^k\|^2 \},
\]
which is equivalent to~\eqref{a-dfsdf} since $(\bI_n)_{S_k}=\bS_k \bS_k^\top$.
\end{proof}

\section{Numerical Experiments}

In our first experiment (Figure~\ref{fig:first}) we compare SDNA and our new minibatch version of SDCA on one real (\verb"mushrooms"; $d=112$, $n=8,124$) and one synthetic ($d=1,024$, $n=2,048$) dataset. In both cases, we used  $\lambda=1/n$ as the regularization parameter and $g(w)=\tfrac{1}{2}\|w\|^2$.  As $\tau$ increases, SDNA requires less passes over data (epochs), while SDCA requires more passes over data. It can be shown that this behavior can be predicted from the complexity results for these two methods. The difference in performance depends on the choice of the dataset and can be quite dramatic. 

\begin{figure}
\includegraphics[width=4cm]{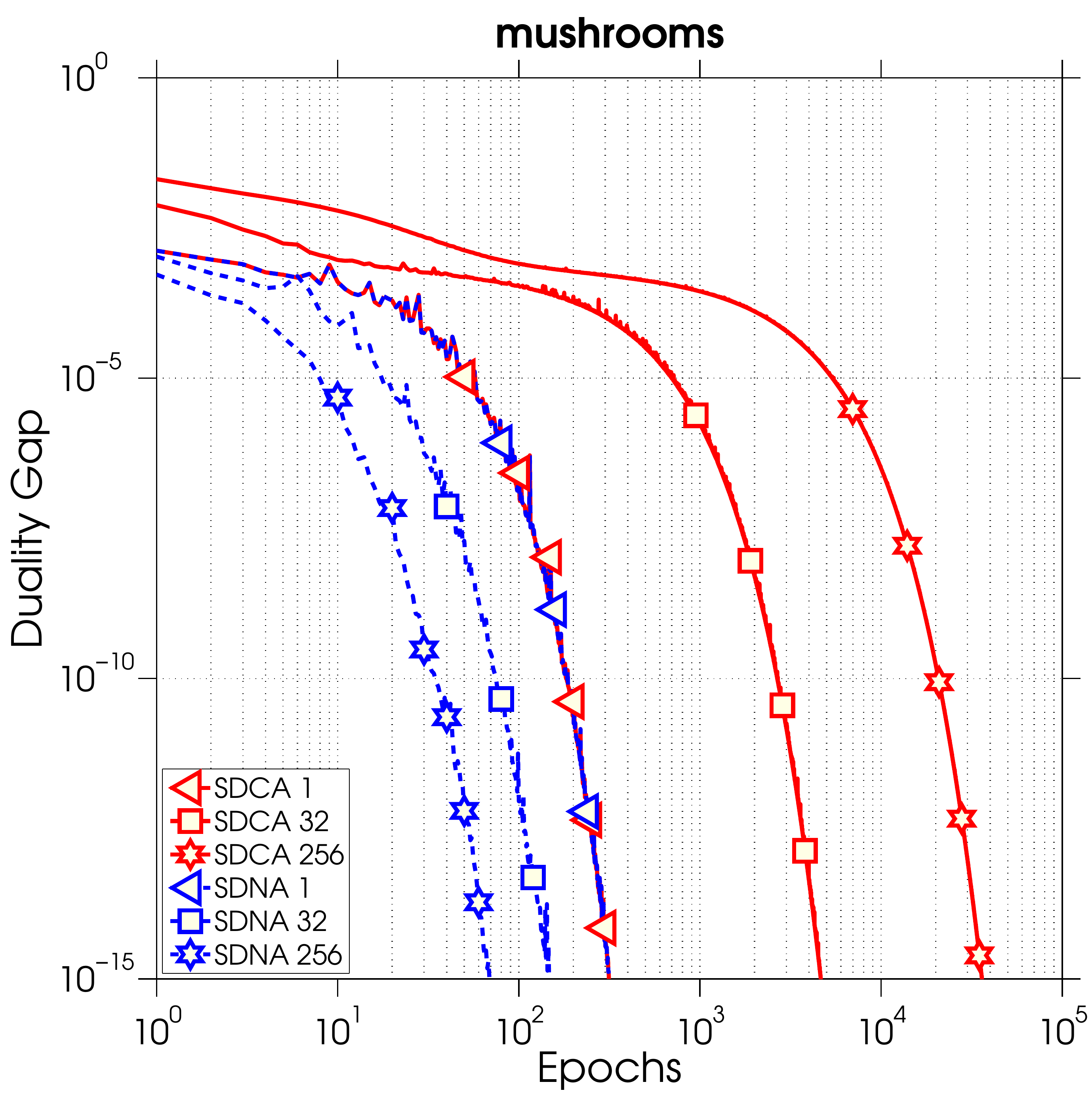} 
\includegraphics[width=4cm]{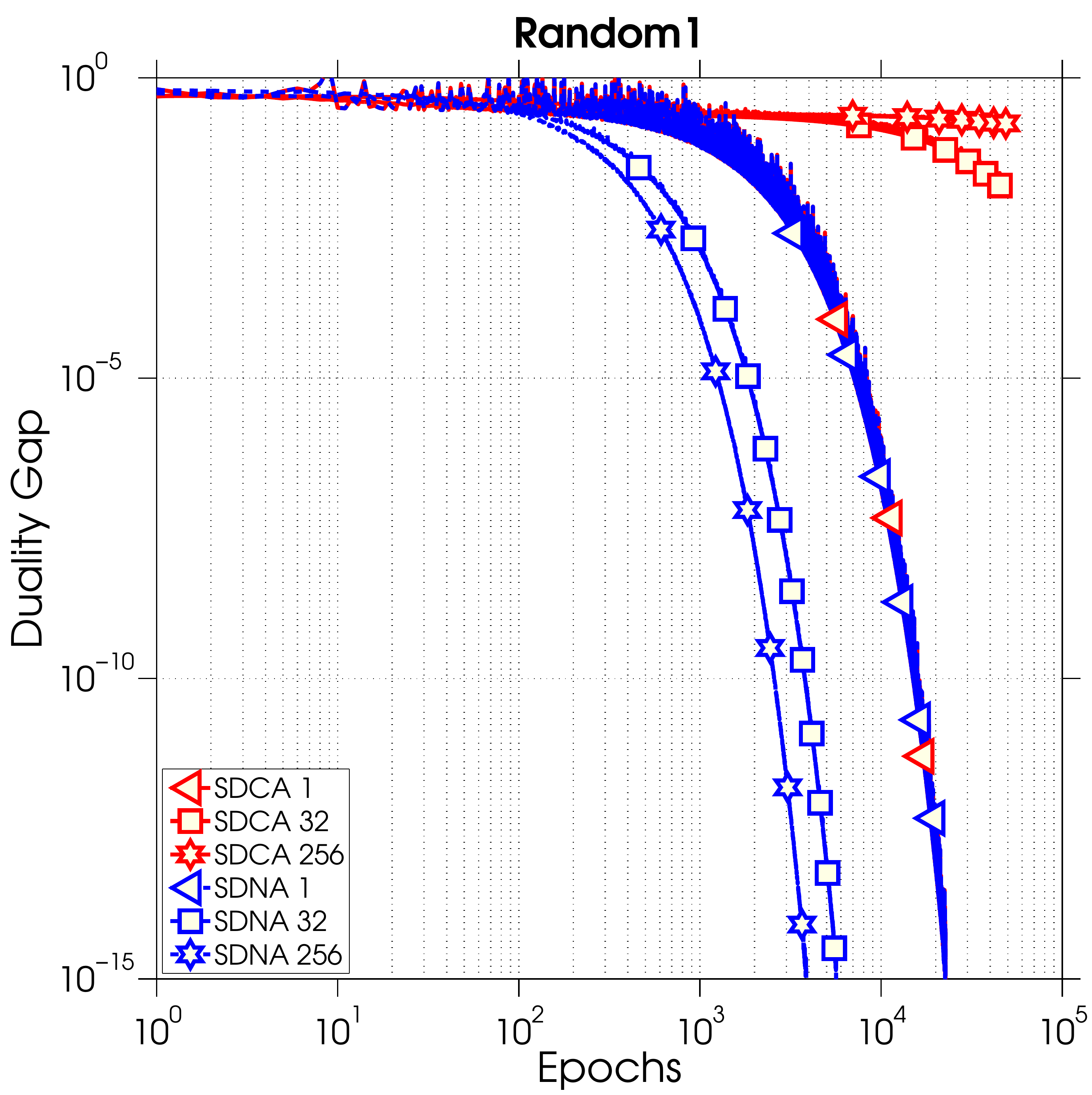}
\caption{\scriptsize Comparison of SDNA and SDCA for minibatch sizes $\tau=1,32,256$ on a real (left) and synthetic (right) dataset. The methods coincide for $\tau=1$. }
\label{fig:first}
\end{figure}

In the second experiment (Figure~\ref{fig:workload}), we investigate how much time it takes for the methods to process a single epoch, using the same datasets as before. As $\tau$ increases, SDNA does more work as the subproblems it needs to solve in each iteration involve a $\tau\times \tau$ submatrix of the Hessian of the smooth part of the dual objective function. On the other hand, the work SDCA needs to do is much smaller, and does nearly does not increase with the minibatch size $\tau$. This is because the subproblems are separable. As before, all experiments are done using a single core (however, both methods would benefit from a parallel implementation).  

\begin{figure}
\includegraphics[width=4cm]{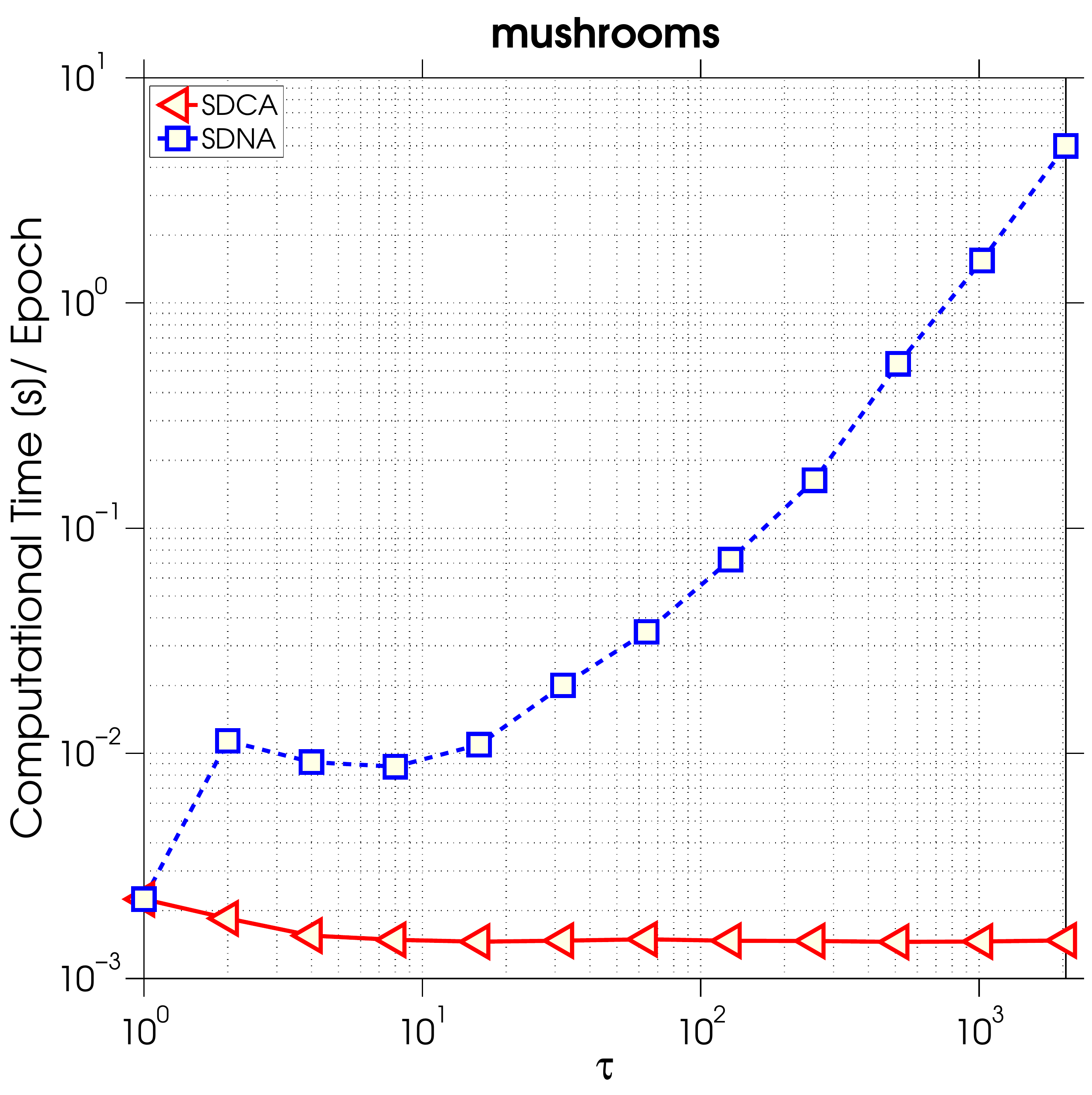} 
\includegraphics[width=4cm]{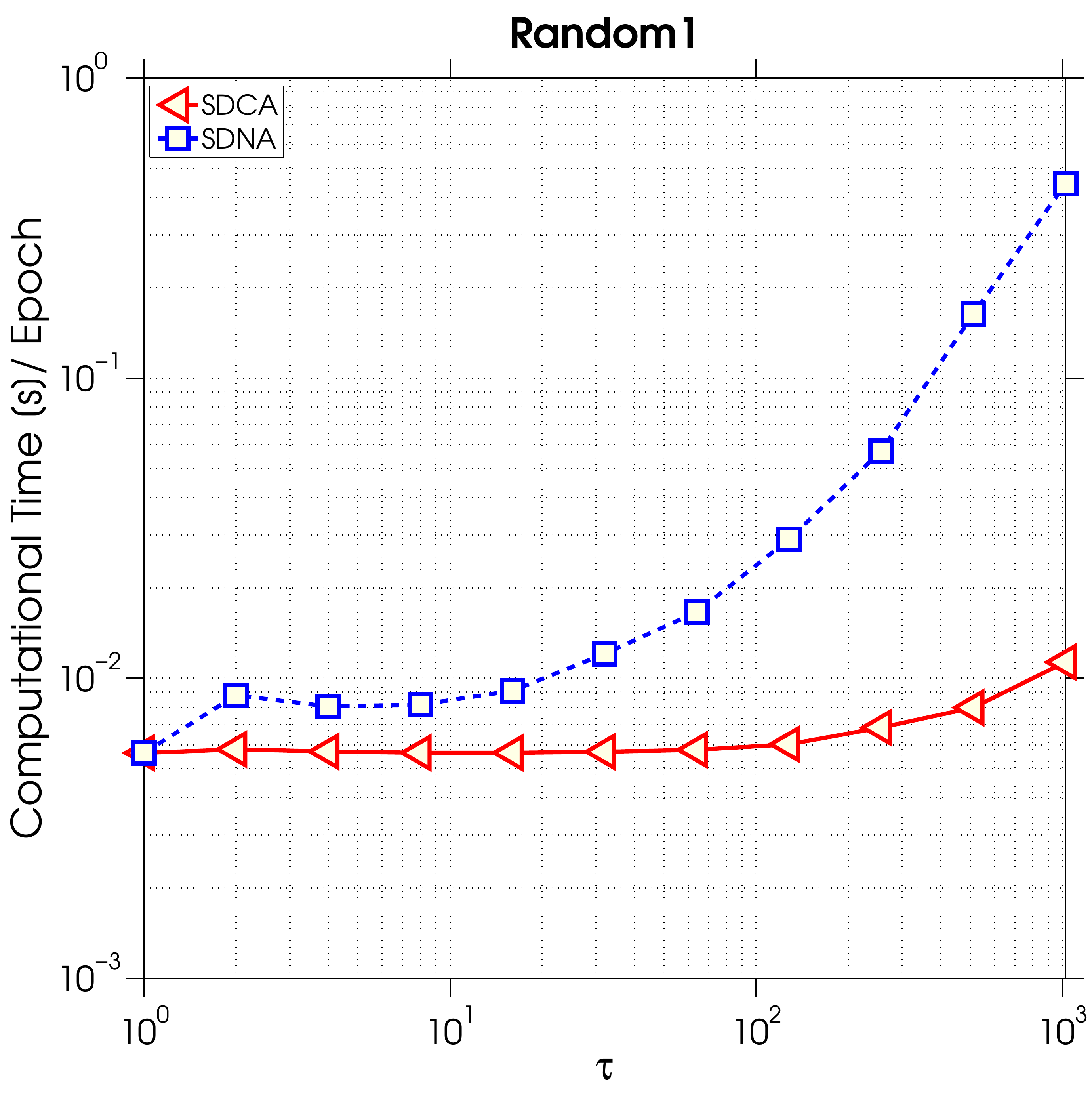}
\caption{
\scriptsize
 Time it takes for SDNA and SDCA to process a singe epoch as a function of the minibatch size $\tau$. }
\label{fig:workload}
\end{figure}

Finally, in Figure~\ref{fig:runtime}  we put the insights gained from the previous two experiments together: we look at the performance of SDNA for various choices of $\tau$ by comparing runtime and duality gap error.  We should expect that increasing $\tau$ would lead to faster method in terms of passes over data, but that this would also lead to slower iterations. The question is, is does the gain outweight the loss? The answer is: yes, for small enough minibatch sizes. Indeed, looking at Figure~\ref{fig:runtime}, we see that the runtime of SDNA improved up to the point $\tau=16$ for both datasets, and then starts to deteriorate. In situations where it is costly to fetch data from memory to a (fast) processor, much larger minibatch sizes would be optimal.

\begin{figure}
\includegraphics[width=4cm]{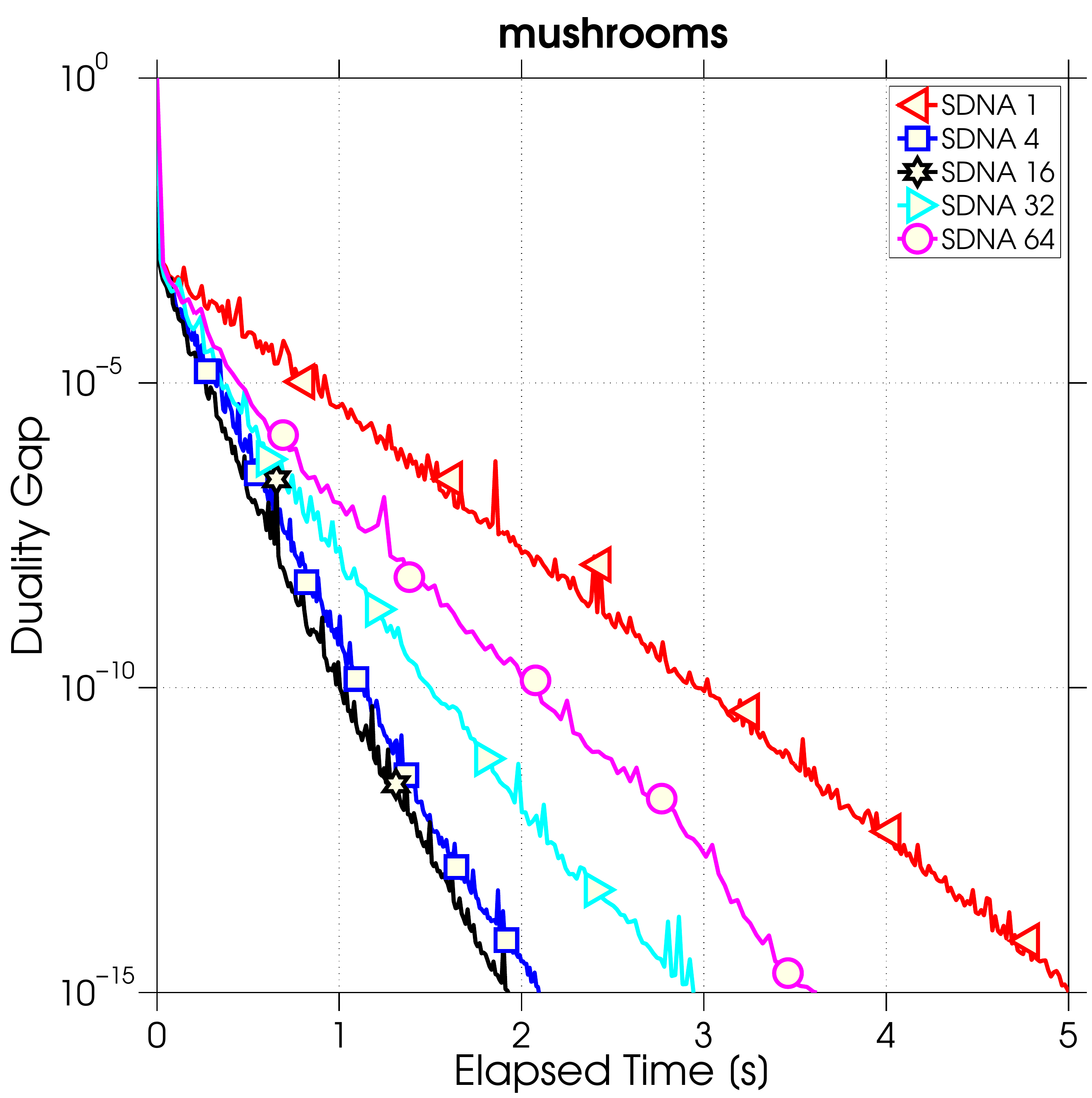} 
\includegraphics[width=4cm]{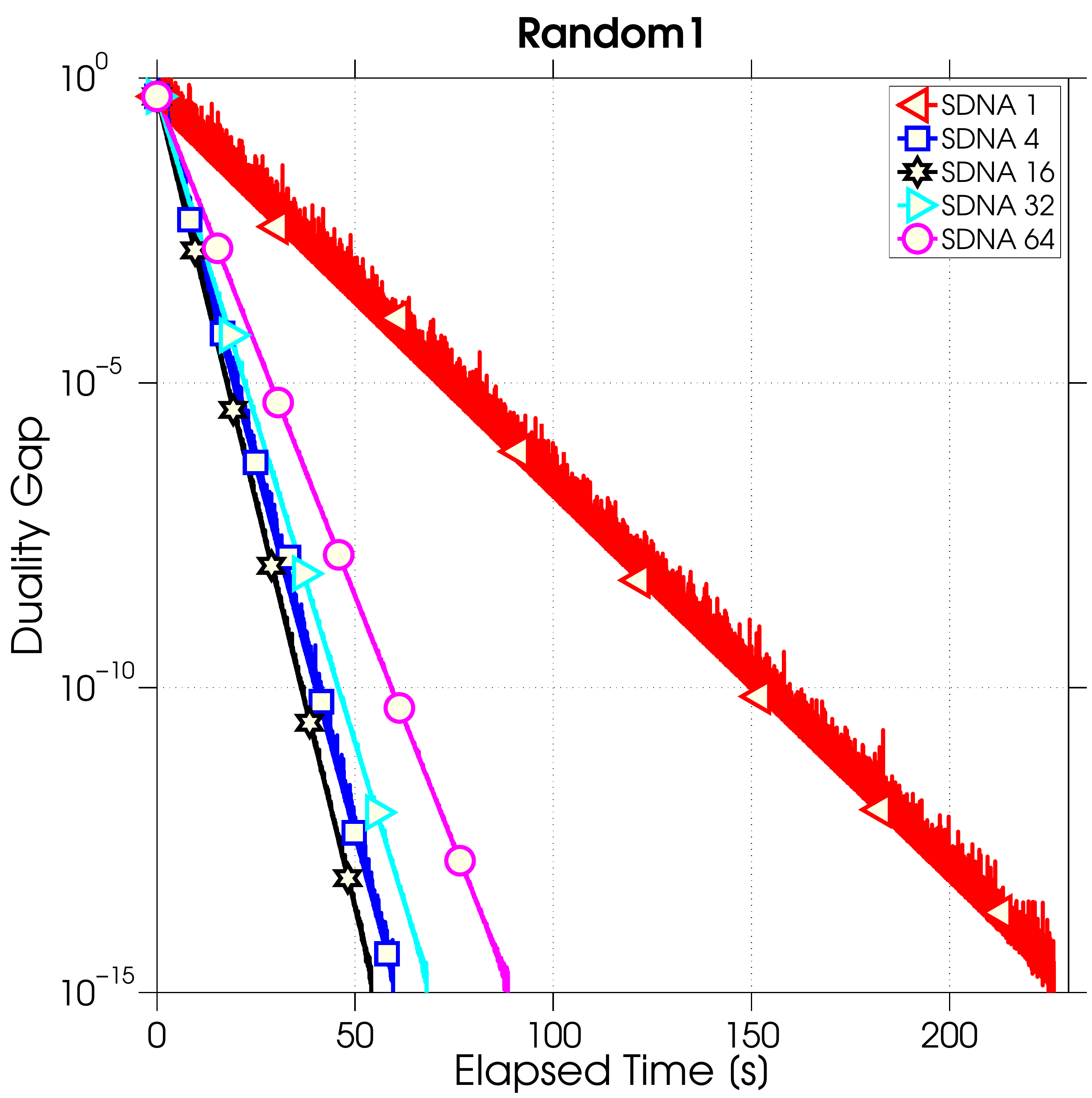}
\caption{\scriptsize Runtime of SDNA for minibatch sizes $\tau=1,4,16,32,64$ .}
\label{fig:runtime}
\end{figure}


\bibliographystyle{icml2015} 
\bibliography{biblio}
 
\clearpage
\onecolumn

\section*{APPENDIX: Proof of Theorem~\ref{prop-proximal}}

It follows directly from Assumption~\ref{ass:smoothness} and the update rule $x^{k+1}=x^k+(h^k)_{S_k}$ in Algorithm~\ref{alg:a1} that:
\begin{eqnarray*}LHS &:=& f(x^{k+1})+\sum_{i=1}^n \psi_i(x^{k+1}_i)-f(x^k) -\sum_{i\notin S_k} \psi_i(x_i^k) \\
&
 \leq & \ve{\nabla f(x^k)}{(h^k)_{S_k}}+\frac{1}{2}\ve{h^k}{\bX_{S_k} h^k} 
+\sum_{i\in S_k}\psi_i(x_i^k+h_i^k).
\end{eqnarray*}

Since $h^k$ is defined as the minimizer of the right hand side in the last inequality, we can further bound this term by replacing $h^k$ with  $h = \lambda(x^*-x^k)$ for  arbitrary $\lambda \in [0,1]$:
\begin{equation}\label{eq:09j09986} LHS
 \leq  \lambda \ve{(\nabla f(x^k))_{S_k}}{x^*-x^k}
+\sum_{i\in S_k}\psi_i(x_i^k+\lambda(x_i^*-x_i^k)) +\frac{\lambda^2}{2}\ve{x^*-x^k}{\bX_{S_k}
(x^*-x^k)} .
\end{equation}
Now we use the fact that $\psi_i$ is $\gamma_i$-strongly convex  to obtain:
\begin{eqnarray*}
F(x^{k+1})-F(x^k) &=&
 f(x^{k+1})+\sum_{i=1}^n \psi_i(x^{k+1}_i)-f(x^k)-\sum_{i=1}^n\psi_i(x_i^k)\\
&\overset{\eqref{eq:09j09986}}{\leq} & \lambda\ve{(\nabla f(x^k))_{S_k}}{x^*-x^k}
+\lambda\sum_{i\in S_k}[\psi_i(x_i^*)-\psi_i(x_i^k)] \\
&& \qquad  -\frac{\lambda(1-\lambda)}{2}
\ve{x^*-x^k}{\bD(\gamma)_{S_k}
(x^*-x^k)} +\frac{\lambda^2}{2}\ve{x^*-x^k}{\bX_{S_k}
(x^*-x^k)}.
\end{eqnarray*}
By taking expectations in $S_k$ on both sides of the last inequality, we see that for any $\lambda \in [0,1]$, the following holds:
\begin{eqnarray*}
 \Exp_k[F(x^{k+1}) - F(x^k)]
& \leq & 
\frac{\lambda\tau}{n}\left(\ve{(\nabla f(x^k))}{x^*-x^k}+ \sum_{i=1}^n \left(\psi_i(x_i^*)-
\psi_i(x_i^k)\right) \right) \\
&& \quad -\frac{\lambda(1-\lambda)}{2}
\ve{x^*-x^k}{\Exp\left[\bD(\gamma)_{\Sam}\right]
(x^*-x^k)} +\frac{\lambda^2}{2}\ve{x^*-x^k}{\Exp\left[\bX_{\Sam}\right]
(x^*-x^k)}\\
& \leq &
\frac{\lambda\tau}{n}\left(F(x^*)-F(x^k)-\frac{1}{2}\ve{x^*-x^k}{\bG(x^*-x^k)}\right) \\ 
&& \quad +\frac{\lambda^2}{2}\ve{x^*-x^k}{\Exp\left[\bX_{\Sam}+\bD(\gamma)_{\Sam}\right]
(x^*-x^k)}- \frac{\lambda}{2}\ve{x^*-x^k}{ \Exp\left[\bD(\gamma)_{\Sam}\right]
(x^*-x^k)}\\
& \leq  &
\frac{\lambda\tau}{n}\left(F(x^*)-F(x^k)\right)-\frac{\lambda}{2}\ve{x^*-x^k}{\frac{\tau}{n}(\bD(\gamma)+\bG)(x^*-x^k)} \\
&& \quad +\frac{\lambda^2}{2}\ve{x^*-x^k}{\Exp\left[\bX_{\Sam}+\frac{\tau}{n}\bD(\gamma)\right]
(x^*-x^k)},
\end{eqnarray*}
where the second to last inequality follows from Assumption~\ref{a:strong_convexity} and in the last one we used the fact that $\Exp[\bD(\gamma)_{\Sam}]=\frac{\tau}{n}\bD(\gamma)$. 
It remains to replace $\lambda$ by $\min(1,s)$.

\section*{APPENDIX: Proof of Lemma~\ref{l-pdgd}}

Recall that $\bM=\tfrac{1}{n}\bX$, where $\bX=\tfrac{1}{\lambda n}\bA^\top \bA$.  

For simplicity in this proof we write $\theta =\theta(\Sam)$.
First, by the 1-strong convexity of the function $g$ we obtain the 1-smoothness of the function $g^*$, from which we deduce:
\begin{align*}
 -\lambda g^*(\bar \alpha^{k+1})+\lambda g^*(\bar \alpha^k)  +\lambda \ve{\nabla g^*(\bar \alpha^k)}{\bar \alpha^{k+1}-
\bar \alpha^k} \geq-\frac{\lambda}{2}\ve{\bar\alpha^{k+1} -\bar \alpha^k}{\bar\alpha^{k+1}-\bar \alpha^k}.
\end{align*}
Now we replace $\nabla g^*(\bar \alpha^k)$ by $w^k$ and $\bar \alpha $ by $\frac{1}{\lambda n} \bA \alpha$ to obtain: 
\begin{eqnarray*}
D(\alpha^{k+1})-D(\alpha^k) &\geq &
\frac{1}{n}\sum_{i\in S_k} \left[-\phi_i^*(-\alpha^{k+1}_i)+\phi_i^*(-\alpha_i^k)\right]  -
\frac{1}{n} \ve{\bA ^\top w^k}{\alpha^{k+1}-\alpha^k} -\frac{1}{2\lambda n^2} (\alpha^{k+1}-\alpha^{k})^\top \bA^\top \bA
(\alpha^{k+1}-\alpha^{k})\\
&=& \max_{h\in \R^n} \left\{ \frac{1}{n}\sum_{i\in S_k} \left[-\phi_i^*(-\alpha_{i}^k-h_i)+\phi_i^*(-\alpha_i^k)\right]  -
\frac{1}{n} \ve{(\bA^\top w^k)_{S_k}}{ h} -\frac{1}{2 n} h^\top \bX_{S_k} h\right\},
\end{eqnarray*}
where in the last equality we used the dual update rules in Algorithm~\ref{alg:pd}, as well as relations~\eqref{eq:hSw}
and~\eqref{eq:hhatS}.  
Therefore, for arbitrary $h\in \R^ n$,
\begin{eqnarray}
 \Exp_k[D( \alpha^{k+1})-D(\alpha^k)]
&\geq &  \Exp_k \left[ \frac{1}{n}\sum_{i\in S_k} \left[-\phi_i^*(-\alpha_{i}^k-h_i)+\phi_i^*(-\alpha_i^k)\right] \right] 
- \Exp_k \left[
\frac{1}{n} \ve{(\bA^\top w^k)_{S_k}}{ h} -\frac{1}{2n} h^\top \bX_{S_k} h\right] \notag \\
& =& \frac{1}{n}\sum_{i=1}^n 
p_i\left[-\phi_i^*(-\alpha_{i}^k-h_i)+\phi_i^*(-\alpha_i^k) - (a_i^\top w^k) h_i\right]  \notag -\frac{1}{2 n}
h^\top \Exp[\bX_{\Sam}]h.  \label{a-htodsf}
\end{eqnarray}
Let $u^k \in \R^n$ such that $u_i^k=\nabla \phi_i(a_i^ \top w^k) \in \R$ for all $ i\in [n]$. Let $s=(s_1,\dots,s_n)\in [0,1]^ n$ with  $s_i=\theta p_i^{-1}$ for all $i\in [n]$, where
 $\theta$ is given in~\eqref{a:theta}.
By using $h_i=-s_i(\alpha^k_i+u^k_i)$ for all $i\in [n]$ in~\eqref{a-htodsf}, we get:
\begin{eqnarray*}
\Exp_k[D(\alpha^{k+1})-D(\alpha^k)]
&\geq  &
 \frac{1}{n}\sum_{i=1}^n 
p_i [-\phi_i^*\left(-(1-s_i)\alpha_i^k+s_i u^k_i\right)+\phi_i^*(-\alpha_i^k) + s_i \ve{a_i^\top w^k}{ \alpha_i^k+u^k_i} ] \\
&& \qquad 
-\frac{1}{2n}
(\alpha^k+u^k)^\top \bD(s) \Exp[\bX_{\Sam}] \bD(s)  (\alpha^k+u^k)
\end{eqnarray*}
From $\gamma$-strong convexity of the functions $\phi_i^*$ we deduce that:
\begin{eqnarray*}
-\phi_i^*((1-s_i)(-\alpha_i^k) + s_i u^k_i )+\phi_i^*(-\alpha_i^k)
 \geq s_i \phi_i^*(-\alpha_i^k)-s_i\phi_i^*(u^k_i ) +\frac{\gamma s_i(1-s_i) }{2}
| u^k_i +\alpha_i^k|^2.
\end{eqnarray*}
Consequently, 
\begin{eqnarray*}
\Exp_k[D(\alpha^{k+1})-D(\alpha^k)]
&\geq & 
 \frac{1}{n}\sum_{i=1}^n 
p_is_i \left[ \phi_i^*(-\alpha_i^k)-\phi_i^*(u^k_i )+  \ve{a_i^\top w^k}{ \alpha_i^k+u^k_i}\right]  
+\frac{1}{n}\sum_{i=1}^ n\frac{\gamma p_i s_i(1-s_i) }{2}
| u^k_i +\alpha_i^k|^2  
\\&&   \qquad-\frac{1}{2 n}
(\alpha^k+u^k)^\top \bD(s)\Exp[\bX_{\Sam}]\bD(s)(\alpha^k+u^k)\\
&=&
\frac{\theta}{n}\sum_{i=1}^n 
\left[ \phi_i^*(-\alpha_i^k)+\phi_i(a_i^\top w^k )  + \ve{a_i^\top w^k}{\alpha_i^k}\right]
+\frac{\gamma \theta}{2n}\ve{
\alpha^k+u^k}{(\bI-\bD(s))(\alpha^k+u^k)}  \\ & & \qquad-\frac{1}{2 n}
\ve{\alpha^k+u^k}{ \bD(s)\Exp[\bX_{\Sam}]\bD(s)(\alpha^k+u^k)}
\end{eqnarray*}
where the equality follows from $
 u_i^k=\nabla \phi_i(a_i^ \top w^k)$.
Next, by the definition of $\theta$ in~\eqref{a:theta}, we know that:
\begin{eqnarray*}
 \gamma \bI &\succeq & \theta \gamma \bD(p^{-1}) +\frac{\theta}{\lambda n}\bD(v\circ p^{-1}) \\
 &=& \gamma \bD(s)+\frac{1}{\theta\lambda n}\bD(s)\bD(v\circ p) \bD(s) \;\; \overset{\eqref{a:PM}}{\succeq} \;\; \gamma \bD(s)+\frac{1}{\theta}\bD(s)\Exp[\bX_{\Sam}] \bD(s).
\end{eqnarray*}
Finally, it follows that
\begin{eqnarray*}
\Exp_k[D(\alpha^{k+1})-D(\alpha^k)]
&\geq & \frac{\theta}{n}\sum_{i=1}^n \left[ \phi_i^*(-\alpha_i^k)+\phi_i(a_i^\top w^k) + 
\ve{a_i^\top w^k}{\alpha_i^k}\right]
\;\;=\;\; \theta (P(w^k)-D(\alpha^k)).
\end{eqnarray*}


\section*{APPENDIX: More insight into the relationship between $\sigma_2$ and $\sigma_3$}

In the main text we have shown that $\sigma_2\geq \sigma_3$, where $\sigma_2$ is the rate of Method 2 and $\sigma_3$ is the rate of Method 3:  NSync \cite{NSync}. In this section we give a more detailed description of the relationship between these two quantities in the case when $\Sam$ is the $\tau$-nice sampling \cite{PCDM}. That is, $\Sam$ picks subsets of $[n]$ of cardinality $\tau$, uniformly at random. For this sampling, 

\[p_i := \Prob(i\in \Sam) =\frac{\tau}{n} .\]

\begin{proposition}
Suppose that $\bG=\bM$ and $\Sam$ be the $\tau$-nice sampling. Then there exists $\beta \in [1, \tau]$ such that one can choose $v_i = \beta \bM_{i,i}$  
 and 
\[
\sigma_2 = \frac{\beta \sigma_3}{(1-\frac{\tau-1}{n-1}) + \frac{n}{\tau}\frac{\tau-1}{n-1}\beta \sigma_3}.
\]
\end{proposition}

\begin{proof}

As explained in~\cite{PCDM}, \eqref{eq:ESO} is always true if we take $v_i = \beta \bM_{i,i}$ with $\beta = \tau$ but 
smaller values (leading to a faster algorithm) may be computable if the problem exhibits a  property called ``partial separability''.

Let us denote by $\bD$ the diagonal matrix whose entries are the diagonal entries of $\bM$.
\begin{align*}
&(\bM_{[\Sam]})_{i,i}=\begin{cases} M_{i,i}=D_{i,i} & \text{ if } i \in \Sam \quad \text{ (probability } \frac{\tau}{n} \text{)} \\ 0 & \text{ otherwise} \end{cases} \\
&(\bM_{[\Sam, \Sam]})_{i,j}=\begin{cases} M_{i,j} & \text{ if } i \in \Sam \text{ and }  j \in \Sam \quad \text{ (probability } \frac{\tau(\tau-1)}{n(n-1)} \text{)}\\ 0 & \text{ otherwise.} \end{cases}
\end{align*}
Hence, 
\[
\mathbf{E}[\bM_{\Sam}]= \frac{\tau}{n} \bD + \frac{\tau(\tau-1)}{n(n-1)}(\bM-\bD) 
= \frac{\tau}{n} \Big( (1-\frac{\tau-1}{n-1} )\bD +\frac{\tau-1}{n-1}\bM\Big).
\]
Let us denote $\bA = \bM^{-1/2} \bD \bM^{-1/2}$ and $\alpha = \frac{\tau-1}{n-1}$.
\[
\sigma_3 \overset{\eqref{eq:sigma3}}{=} \frac{\tau}{n} \beta^{-1} \lambda_{\min}(  \bM^{1/2} \bD^{-1} \bM^{1/2}) = \frac{\tau}{n} \beta^{-1} \big(\lambda_{\max}( \bA)\big)^{-1}
\]
\begin{align*}
\sigma_2 &\overset{\eqref{eq:sigma2}}{=}  \frac{\tau^2}{n^2} \lambda_{\min}(  \bM^{1/2} (\mathbf{E}[\bM_{\Sam}])^{-1} \bM^{1/2}) = \frac{\tau^2}{n^2} \lambda_{\min}(  \bM^{1/2} \frac{n}{\tau}((1-\alpha) \bD + \alpha \bM)^{-1} \bM^{1/2})  \\
& \; = \; \frac{\tau}{n} \big(\lambda_{\max}(  \bM^{-1/2} ((1-\alpha) \bD + \alpha \bM) \bM^{-1/2})\big)^{-1} = \frac{\tau}{n} \big(\lambda_{\max}( (1-\alpha) \bA + \alpha \bI )\big)^{-1}
\end{align*}

But we have
$\lambda_{\max}((1- \alpha) \bA + \alpha \bI ) = (1-\alpha) \lambda_{\max}(\bA) + \alpha$, so
\begin{align*}
& \frac{\tau}{n \sigma_2} = (1-\alpha) \frac{\tau}{n \beta \sigma_3} +  \alpha \\
& \frac{n \sigma_2}{\tau} = \frac{1}{ (1-\alpha)\frac{\tau}{n \beta \sigma_3} + \alpha } \\
& \sigma_2= \frac{\sigma_3}{(1-\alpha) \beta^{-1} +\alpha \frac{n}{\tau}\sigma_3} =  \frac{\beta\sigma_3}{ (1-\frac{\tau-1}{n-1}) +\frac{\tau-1}{n-1} \frac{n}{\tau}\beta \sigma_3} 
\end{align*} 

Note that if $\sigma_3$ is small, then $\sigma_2$ is of the order of $ \frac{\beta \sigma_3}{1-\frac{\tau-1}{n-1} } > \beta \sigma_3$ .
\end{proof}

\end{document}